\documentclass[conference]{IEEEtran}
\IEEEoverridecommandlockouts
\usepackage{cite}
\usepackage{amsmath,amssymb,amsfonts}
\usepackage{amsthm}
\newtheorem{theorem}{Theorem}

\usepackage{graphicx}
\usepackage{booktabs}
\usepackage{textcomp}
\usepackage{xcolor}
\usepackage{url}

\usepackage{hyperref}
\graphicspath{{figures/}{../figures/}}
\hypersetup{hidelinks}
\title{M-CTX: Exact and Scalable Spatial Context\\Retrieval for Trajectory Analytics}
\author{
\IEEEauthorblockN{
Kun Ma\IEEEauthorrefmark{1}\IEEEauthorrefmark{2},
Qilong Han\IEEEauthorrefmark{1},
Chengjing Song\IEEEauthorrefmark{1},
Jingzheng Yao\IEEEauthorrefmark{1},
Xiao Han\IEEEauthorrefmark{1},
Yuee Zhou\IEEEauthorrefmark{3}\IEEEauthorrefmark{2},
Changmao Wu\IEEEauthorrefmark{4}
}
\IEEEauthorblockA{\IEEEauthorrefmark{1}
Harbin Engineering University, Harbin, China\\
}
\IEEEauthorblockA{\IEEEauthorrefmark{2}
Politecnico di Torino, Turin, Italy}
\IEEEauthorblockA{\IEEEauthorrefmark{3}
Northeastern University, Shenyang, China\\
}
\IEEEauthorblockA{\IEEEauthorrefmark{4}
Chinese Academy of Sciences, Beijing, China\\
}
}

\begin{document}
\maketitle

\begin{abstract} Modern trajectory predictors increasingly condition on external spatial context, such as map geometry, signed distance fields (SDFs), and nearby moving agents. While this context improves prediction quality, constructing it for every training anchor has become a hidden systems bottleneck. In a representative maritime AIS pipeline, spatial context construction requires roughly $17$ CPU-days for a $5.48$M-anchor corpus, dominating the cost of the downstream predictor. We present \textbf{M-CTX}, an exact and scalable spatial context-retrieval framework for trajectory analytics. M-CTX recasts context construction as an ingest-once, query-many spatial database workload and replaces three brute-force stages---OSM range retrieval, SDF computation, and moving-vessel neighbour lookup---with composable, index-backed operators. Its learned range-index backend, BR-LZ, provides recall-complete MBR-overlap range retrieval and reduces candidate amplification by $1.1\times$--$2.7\times$ relative to global-expansion one-curve baselines. Across four maritime regions, eight baseline systems, synthetic workloads with up to $40$M spatial features, and $10^7$-record AIS streams, M-CTX reproduces the reference context exactly. On the $5.48$M-anchor corpus, it reduces context construction from about $17$ CPU-days to $1.8$ hours, a measured $226\times$ end-to-end speed-up. An optional storage mode further compresses SDF context by $64\times$ with only a $0.04$\,m ADE change. These results establish exact spatial context retrieval as a first-class database problem in modern trajectory analytics. Code and datasets are publicly available at https://github.com/mark000071/M-CTX-Traj. \end{abstract}


\section{Introduction}\label{sec:intro}

Trajectory prediction is increasingly context-aware. In autonomous
driving, urban mobility, and maritime navigation, plausible future motion
depends not only on an agent's recent trajectory, but also on surrounding
map geometry, physical constraints, and nearby agents. We study this shift
in maritime traffic, where public Automatic Identification System (AIS)
feeds provide one of the largest open sources of moving-object data.
Recent maritime predictors increasingly condition on coastlines, channels,
signed distance fields (SDFs), and neighbouring vessels. These signals improve prediction quality, but
they also introduce a new requirement: spatial context must be retrieved
and materialised for every training sample.

This context construction is often treated as routine preprocessing, but
in practice it dominates the pipeline. For each AIS anchor, the reference
pipeline performs three operations: (i) it retrieves OpenStreetMap (OSM)
features overlapping a $5$\,km window; (ii) it rasterises the local map
patch onto a $128\times128$ grid and computes two SDFs; and (iii) it scans
a vessel snapshot to find neighbours within $3$\,km. Profiling on a
$1{,}000$-anchor subset (Section~\ref{sec:problem}) shows that SDF
construction alone accounts for about half of the wall-clock time, mostly
due to a quadratic distance transform. Brute-force neighbour scanning adds
another quarter. Extrapolated to the full $5.48$M-anchor corpus, context
construction requires roughly $17$ CPU-days before model training begins.
The predictor is not the bottleneck; the context pipeline is. Figure~\ref{fig:intro} summarises this motivation and the database
reformulation behind M-CTX.
\begin{figure}[t]
\centering
\includegraphics[width=\columnwidth]{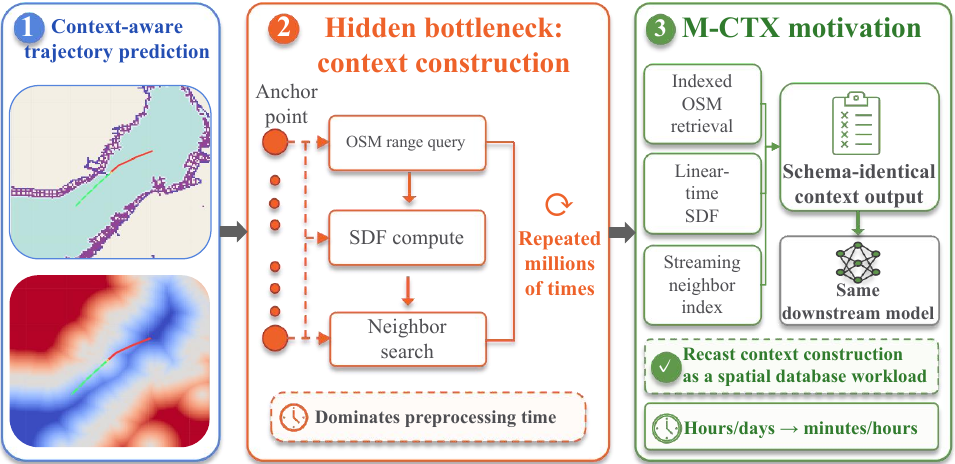}
\caption{M-CTX recasts per-anchor spatial context construction as an
exact, index-backed database workload.}
\label{fig:intro}
\end{figure}
Our key observation is that this bottleneck is a spatial-database workload
implemented without spatial data management. OSM retrieval is a static
range-query problem over map features, the canonical workload of spatial
indexes such as R-trees~\cite{guttman1984rtree}. Neighbour retrieval is a
streaming moving-object query, matching the setting of indexes such as the
TPR-tree and B$^x$-tree~\cite{saltenis2000tpr,jensen2004bx}. SDF
construction repeatedly computes exact distance transforms over local
rasters, for which separable linear-time algorithms are known
~\cite{felzenszwalb2004distance}. Although these operations appear in a
learning pipeline, their common structure is database-centric: millions of
anchors repeatedly request exact spatial context from largely reusable
static and streaming data.

We present \textbf{M-CTX}, an exact and scalable spatial context-retrieval
framework for trajectory analytics. M-CTX replaces the three brute-force
stages with composable, index-backed operators: an exact OSM range
retriever, a linear-time SDF engine, and an incremental moving-object
neighbour index. The framework preserves the schema and semantics of the
reference context, allowing downstream models to consume the generated
tensors without code changes or retraining. Its learned range-index
backend, BR-LZ, uses bounded residuals and a segment-local extent to
provide recall-complete MBR-overlap range retrieval while reducing
candidate amplification relative to global expansion strategies.

This paper makes the following contributions.

\begin{itemize}
\item \textbf{A spatial context-retrieval workload.}
We formalise context construction for trajectory learning as an
ingest-once, query-many spatial database workload and show through
profiling that it dominates a representative AIS learning pipeline.

\item \textbf{BR-LZ, a recall-complete learned range index.}
We introduce a bounded-residual learned $Z$-order index for MBR-overlap
range queries. BR-LZ provides an explicit recall-completeness guarantee
and reduces candidate amplification by $1.1\times$--$2.7\times$ over a
global expansion strategy at the same recall.

\item \textbf{Exact SDF acceleration and storage co-design.}
We replace the quadratic SDF kernel with a linear-time exact distance
transform, accelerating the SDF stage by $163\times$. We further show that
SDF context can be compressed by $64\times$ with only a $0.04$\,m ADE
change.

\item \textbf{An end-to-end exact context framework.}
Across four maritime regions, eight baseline systems, synthetic scaling to
$40$M map features, and $10^7$ streaming AIS records, M-CTX preserves the
reference context semantics. On the $5.48$M-anchor corpus, it reduces
context construction from about $17$ CPU-days to $1.8$ hours, a measured
$226\times$ end-to-end speed-up, while preserving downstream predictions
for the deterministic LSTM+Env-SDF core.
\end{itemize}

M-CTX shows that the retrieval workloads hidden inside modern
spatiotemporal learning pipelines are first-class database problems. By
combining exact spatial indexing, linear-time context materialisation, and
drop-in compatibility with downstream predictors, it removes the dominant
data bottleneck without changing the learning model itself.

\section{Related Work}\label{sec:related}

\paragraph*{Spatial indexing and geospatial data systems}
Spatial indexing is a core topic in database systems. The
R-tree~\cite{guttman1984rtree} and its variants, including the
R$^{+}$-tree~\cite{sellis1987rplus}, Hilbert
R-tree~\cite{kamel1994hilbert}, and R$^{*}$-tree~\cite{beckmann1990rstar},
remain standard access methods for rectangle and region queries. Main-memory
structures such as kd-trees~\cite{bentley1975kdtree} and
quadtrees~\cite{finkel1974quadtree} provide complementary design points,
while GiST~\cite{hellerstein1995gist} abstracts many spatial access
methods behind a common interface. For static, read-mostly workloads,
bulk loading is often decisive: the Sort-Tile-Recursive algorithm
~\cite{leutenegger1997str} produces compact R-tree layouts and is widely
used in practice. More recent work has revisited spatial indexing under
dynamic and workload-aware settings, such as Waffle~\cite{moti2022waffle},
which adapts disk-based spatial indexing to query/update trade-offs.
Large-scale geospatial systems, including SpatialHadoop
~\cite{eldawy2015spatialhadoop}, GeoSpark/Sedona~\cite{yu2015geospark},
and RDPro~\cite{shang2024rdpro}, target distributed vector or raster
analytics. M-CTX is complementary: it targets a lighter but repeated
training-loop workload, where the same static map and streaming AIS state
are queried millions of times to materialise exact tensors for downstream
models.

\paragraph*{Learned and workload-aware indexes}
Learned indexes replace parts of a comparison-based access method with a
model of the key distribution~\cite{kraska2018case}. One-dimensional
learned indexes, including FITing-Tree~\cite{galakatos2019fiting},
PGM-index~\cite{ferragina2020pgm}, RadixSpline~\cite{kipf2020radixspline},
ALEX~\cite{ding2020alex}, and SALI~\cite{ge2023sali}, have shown that
distribution-aware models can reduce lookup cost or memory footprint under
suitable workloads. Recent evaluations and surveys further show that
learned indexes are not universally superior; their effectiveness depends
on workload, update pattern, model error, and implementation cost
~\cite{sun2023learned_eval,liu2025multidim_eval,almamun2025survey}.
Multi-dimensional learned indexes extend this idea by either projecting
objects onto a space-filling curve~\cite{wang2019learnedz}, partitioning
space with local models as in LISA~\cite{li2020lisa}, learning recursive
partitions as in RSMI~\cite{qi2020rsmi}, or jointly optimising layout and
partitioning as in Flood~\cite{nathan2020flood},
Tsunami~\cite{ding2020tsunami}, Qd-tree~\cite{yang2020qdtree}, and
LMSFC~\cite{gao2022lmsfc}.

M-CTX differs from this line of work in its correctness requirement.
Context construction requires recall-complete MBR-overlap retrieval: a
feature whose MBR overlaps the query window must be returned even if its
centroid lies outside the window. Centroid-based one-curve learned indexes
can miss such features unless the query is expanded. A global maximum
extent restores recall but produces loose candidate sets. BR-LZ addresses
this gap directly by combining bounded residuals with segment-local
extents, providing recall-complete MBR-overlap retrieval while reducing
candidate amplification relative to global expansion.

\paragraph*{Moving-object indexing and streaming neighbour retrieval}
The neighbour-retrieval stage is related to moving-object indexing. The
TPR-tree~\cite{saltenis2000tpr} indexes time-parameterised MBRs and
supports predictive queries over moving objects. The B$^{x}$-tree
~\cite{jensen2004bx} maps time-partitioned positions onto a space-filling
curve and indexes them with a B$^{+}$-tree, avoiding full snapshot rebuilds
under frequent updates. This property is well matched to AIS streams, where
positions arrive at a regular cadence and neighbour queries are issued
against rolling snapshots. M-CTX adopts a B$^{x}$-tree-style backend with
an exact radius filter after key-range enumeration. The key difference from
snapshot KD-tree baselines is update granularity: M-CTX supports interleaved
inserts and queries, which is required for live context construction.

\paragraph*{Distance transforms and spatial context materialisation}
Signed distance fields are widely used to encode geometric constraints such
as map boundaries, navigable areas, and object surfaces. Exact Euclidean
distance transforms can be computed in linear time by the
Felzenszwalb--Huttenlocher algorithm~\cite{felzenszwalb2004distance};
classical alternatives such as the Maurer transform~\cite{maurer2003edt}
share the same $\Theta(HW)$ bound. Neural signed-distance fields such as
DeepSDF~\cite{park2019deepsdf} learn implicit shape representations, while
geospatial systems such as RDPro~\cite{shang2024rdpro} focus on scalable
raster analytics. M-CTX does not learn the SDF and does not target general
raster analytics. Instead, it replaces the quadratic per-anchor SDF kernel
in a trajectory-learning pipeline with an exact linear-time transform, then
co-designs the stored representation with downstream accuracy.

\paragraph*{Context-aware trajectory prediction}
Recent trajectory predictors increasingly condition on static scene context
and dynamic agent interactions. In autonomous driving, transformer-based
models such as MTR~\cite{shi2022mtr}, Wayformer~\cite{nayakanti2023wayformer},
QCNet~\cite{zhou2023qcnet}, and Forecast-MAE~\cite{cheng2023forecastmae}
explicitly encode road topology, map elements, neighbouring agents, or
masked lane/trajectory context. Earlier general-purpose predictors such as
Trajectron++~\cite{salzmann2020trajectron} and AgentFormer
~\cite{yuan2021agentformer} similarly established the importance of social
and environmental context. Maritime prediction has followed the same trend:
AIS-based models include recurrent and transformer predictors
~\cite{nguyen2021traisformer,zhao2021bilstm,qiang2023mstformer}, graph
models for vessel interaction and traffic flow
~\cite{murray2021ais,liang2022finegrained}, and broader surveys of maritime
trajectory analytics~\cite{tu2018ais,zhang2022survey}. M-CTX is orthogonal to these predictors: it does not
propose a new forecasting model, but accelerates the exact context-retrieval
pipeline on which such models increasingly depend.

\paragraph*{Positioning}
The main distinction of M-CTX is not simply learned versus classical
indexing. Classical spatial indexes provide exact retrieval but may be
heavier to build, store, or embed inside lightweight training pipelines.
Learned spatial indexes improve footprint or build time, but for
MBR-overlap context retrieval they typically rely on empirical recall or
global expansion. M-CTX asks whether a learned spatial index can retain the
deployment advantages of learned indexing while restoring the exactness
required by model-facing retrieval. BR-LZ answers this question for static
OSM range retrieval, and M-CTX combines it with exact SDF materialisation
and streaming neighbour indexing to form an end-to-end, schema-preserving
context retrieval system for trajectory analytics.

\section{Problem Formulation and Workload Characterization}
\label{sec:problem}

\subsection{Context retrieval as a workload}
Context-aware trajectory prediction augments each AIS \emph{anchor}
$a=(\mathrm{lat},\mathrm{lon},t)$ with local spatial context before it is
fed to a predictor. Let $F=\{f_1,\ldots,f_N\}$ be a static collection of
OSM features, where each feature has a two-dimensional MBR and a semantic
tag, and let $S$ be a stream of timestamped vessel positions. For each
anchor $a$, the context pipeline computes
\begin{equation}
\mathcal{C}(a)=
\bigl(
\underbrace{R_F(a,r_1)}_{\text{OSM range}},
\underbrace{D(a)}_{\text{SDF}},
\underbrace{K_S(a,r_2,k)}_{\text{$k$NN}}
\bigr),
\end{equation}
where
$R_F(a,r_1)=\{f\in F: f.\mathrm{MBR}\cap B(a,r_1)\neq\emptyset\}$
returns OSM features intersecting the $r_1=5$\,km query window,
$D(a)\in\mathbb{R}^{H\times W\times 2}$ is a pair of
$128\times128$ signed distance fields over the local patch, and
$K_S(a,r_2,k)$ returns the $k$ nearest vessels within $r_2=3$\,km in the
corresponding snapshot. Since a training corpus evaluates
$\mathcal{C}$ over millions of anchors, the workload is
\emph{ingest-once, query-many}: static map data and historical AIS streams
are prepared once and queried repeatedly.

Figure~\ref{fig:ais} illustrates the spatial skew of AIS traffic. Vessel
positions concentrate along narrow shipping lanes, making neighbour
retrieval both dense and non-uniform---a poor fit for repeated full-snapshot
scans and a natural fit for incremental moving-object indexing.

\begin{figure}[t]
\centering
\includegraphics[width=0.84\columnwidth]{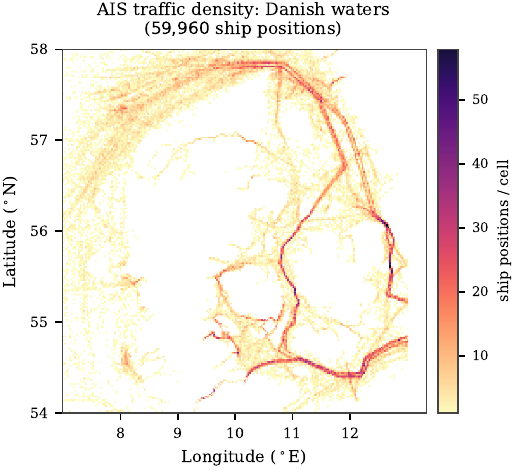}
\caption{AIS traffic density over Danish waters. The lane structure makes
per-anchor neighbour retrieval dense and highly skewed.}
\label{fig:ais}
\end{figure}

\subsection{Profiling the reference pipeline}
Table~\ref{tab:profile} profiles the reference context pipeline on a
$1{,}000$-anchor subset. The retrieval stages dominate the preprocessing
budget. SDF construction is the largest component, taking
$176.4$\,ms per anchor, with most of its time spent in a quadratic
distance-transform kernel. Brute-force neighbour scanning adds another
$88.2$\,ms per anchor. Extrapolated to the full $5.48$M-anchor corpus,
context construction requires approximately $17$ CPU-days before model
training begins. In this pipeline, the predictor is not the bottleneck;
context materialisation is.

\begin{table}[t]
\centering
\caption{Per-stage cost of the reference context pipeline on a
$1{,}000$-anchor subset using a single CPU thread.}
\label{tab:profile}
\small
\setlength{\tabcolsep}{3.5pt}
\begin{tabular}{lrrl}
\toprule
Stage & wall-clock & per-anchor & dominant cost \\
\midrule
OSM load + parse       & $85.2$\,s  & $85.2$\,ms  & cold file I/O \\
SDF compute            & $176.4$\,s & $176.4$\,ms & quadratic EDT \\
Neighbour scan         & $88.2$\,s  & $88.2$\,ms  & full-snapshot scan \\
\bottomrule
\end{tabular}
\setlength{\tabcolsep}{6pt}
\end{table}

\subsection{Why brute force is the wrong abstraction}
The reference implementation treats each context component as an
independent preprocessing kernel, but all three expose reusable spatial
structure. OSM retrieval is a static range query over map features, yet
the reference path reloads and parses tiles per anchor. SDF construction
runs an $O(HW\cdot M)$ nearest-pixel loop even though the exact distance
field can be computed in $O(HW)$. Neighbour retrieval scans the current
snapshot for every anchor instead of indexing moving points. The common
failure mode is therefore not a slow model, but the absence of spatial
data management in a workload that repeatedly asks for exact local context.

This observation leads to the central question of the paper: how can we
replace brute-force context construction with indexed and asymptotically
better operators while preserving the exact tensors consumed by the
downstream model? M-CTX answers this question by recasting context
construction as a spatial database workload embedded inside a learning
pipeline.

\subsection{Design goals and metric}
\label{ssec:goals}
An indexed context backend must satisfy three requirements.

\textbf{Correctness.}
Each operator must reproduce the brute-force result. For OSM range
retrieval, this means provable recall for MBR-overlap queries rather than
empirical high recall.

\textbf{Amortised efficiency.}
The relevant cost is not query latency alone, but the end-to-end
ingest-once, query-many cost over the full corpus, including build time,
query time, and index footprint.

\textbf{Drop-in equivalence.}
The emitted context must preserve the reference schema and numerical
semantics, so downstream predictors can run without code changes or
retraining.

For the OSM range stage, we additionally measure candidate amplification:
\begin{equation}
\alpha(q)=
\frac{
|\{\text{features scanned by the exact filter for } q\}|
}{
|R_F(q)|
},
\end{equation}
defined for queries with $|R_F(q)|>0$. It measures how many exact MBR
tests are performed per true hit after the index has generated its
candidate set. Any recall-complete index has $\alpha(q)\geq1$; reducing
$\alpha$ at recall $1.0$ is therefore the main efficiency lever for
learned range retrieval. Section~\ref{ssec:extent} shows how BR-LZ reduces
this amplification using segment-local extents.

\section{The M-CTX Framework}
\label{sec:overview}

\begin{figure*}[t]
\centering
\includegraphics[width=\textwidth]{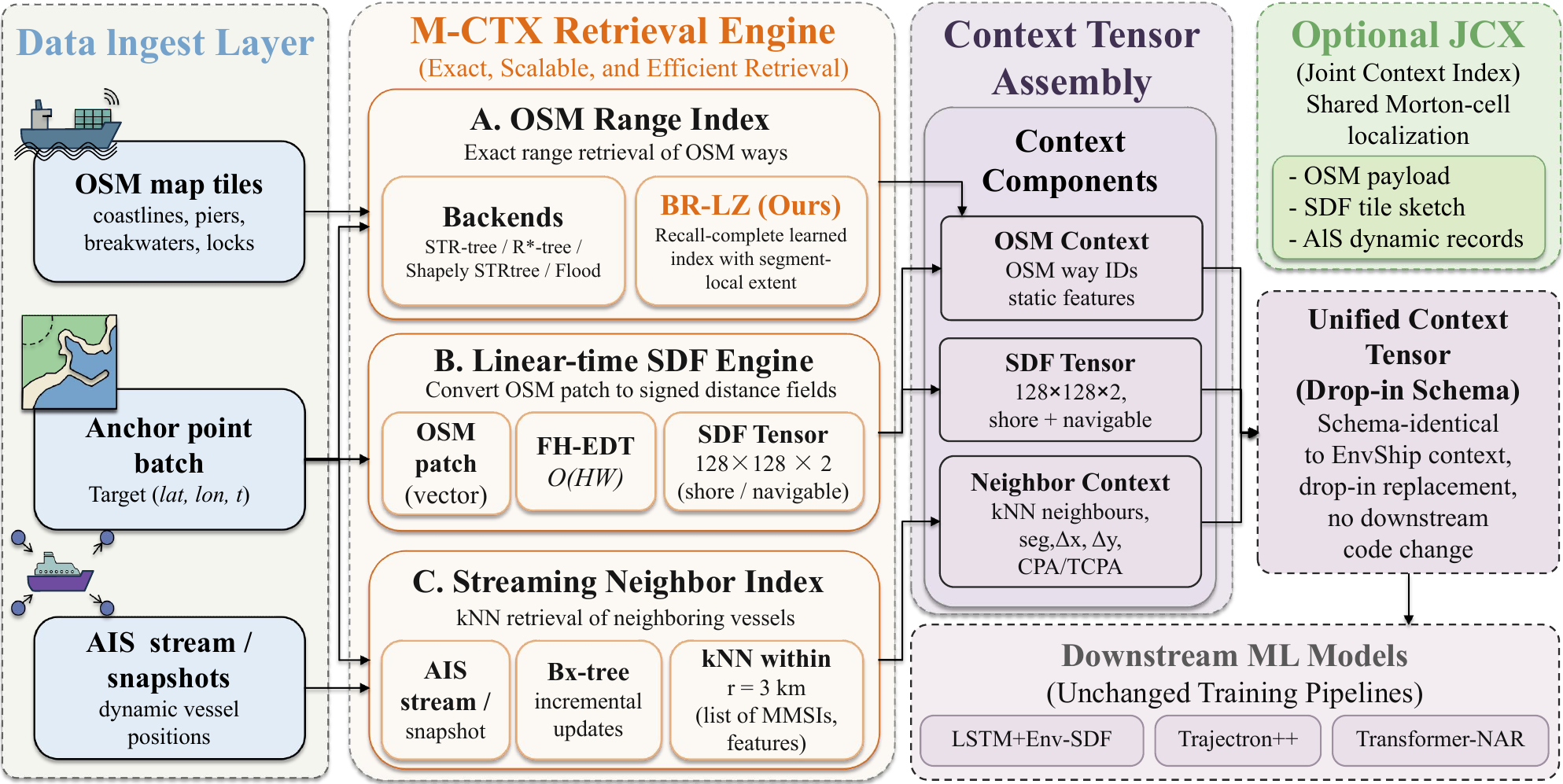}
\caption{M-CTX architecture. An anchor batch is dispatched to three
composable operators: exact OSM range retrieval, linear-time SDF
materialisation, and streaming neighbour lookup. The output follows the
reference context schema, so downstream models run unchanged.}
\label{fig:pipeline}
\end{figure*}

M-CTX implements the workload $\mathcal{C}$ as a composition of three
specialised context operators behind a single interface. The
\textbf{OSM-Index} answers $R_F$ using interchangeable exact backends,
including a pure-NumPy STR-tree, a production R$^{*}$-tree, and the BR-LZ
learned index. The \textbf{SDF engine} computes $D$ with a linear-time
exact distance transform. The \textbf{Neighbour-Index} answers $K_S$ using
an incremental B$^{x}$-tree over streaming AIS positions. Since the OSM
range query depends only on location, the SDF depends on the local map
patch, and neighbour retrieval depends on spatiotemporal state, the three
operators can be executed independently across anchors. Figure~\ref{fig:pipeline} shows the end-to-end M-CTX architecture and how
the three operators are assembled into a schema-identical context output.

M-CTX is invoked through one entry point:
\begin{verbatim}
mctx.build_context(
  anchors, osm_index,
  sdf_engine, neighbor_index)
\end{verbatim}
It returns a context dictionary whose schema matches the reference output.
This drop-in property is a core design requirement: each brute-force stage
can be replaced by an indexed operator while preserving the tensors passed
to the downstream predictor. As a result, the system improves context
construction without changing the learning model or its training code.

The framework is also naturally parallel. The implementation vectorises
operator calls over anchor batches and partitions anchors across workers
with a \texttt{mapPartitions}-style driver. Each worker uses a
broadcast-loaded copy of the static indexes, avoiding repeated rebuilds
when scaling to more processes or machines. This execution model preserves
the same operator composition from single-core experiments to multi-worker
corpus construction.

\section{Static Spatial Range Retrieval with BR-LZ}
\label{sec:static-range}
\label{sec:osm}
\label{sec:brlz}

The OSM stage answers $R_F(a,5\text{\,km})$ over a static collection of
vector map features, including coastlines, breakwaters, piers, and locks.
This is a read-mostly MBR-overlap range-retrieval workload: the map is
ingested once, while millions of anchor-centred windows repeatedly query
the same feature collection. M-CTX exposes three interchangeable exact
backends behind one interface. A pure-NumPy STR-tree
~\cite{leutenegger1997str} provides a dependency-free baseline; a
libspatialindex R$^{*}$-tree~\cite{libspatialindex} provides a tuned
classical reference; and BR-LZ provides a compact learned backend with a
formal recall-completeness guarantee for MBR-overlap queries. All backends
apply an exact MBR filter before returning results, so they preserve the
reference range semantics. This section focuses on BR-LZ.

\subsection{The recall gap in one-curve learned indexes}
One-curve learned spatial indexes are attractive for OSM retrieval because
they map two-dimensional objects to a one-dimensional key order and learn a
small model over that order. They are simple to serialise, cheap to build,
and lightweight enough to broadcast across workers. However, centroid-based
one-curve indexes have a correctness gap for range retrieval: a feature can
overlap the query window even when its centroid lies outside the window.
If the index only searches centroids inside the query range, such a feature
is silently missed.

A standard fix is to expand every query by the global maximum feature
extent before mapping it to the curve. This restores recall but can produce
a loose candidate set, especially when a few long features force all
queries to pay the global worst case. BR-LZ closes the same recall gap with
bounded rank residuals and a segment-local extent. The result is a learned
range index that is recall-complete by construction while reducing
candidate amplification relative to global expansion.

\subsection{Index construction and query}
\label{ssec:brlz-definition}

Let $F=\{f_1,\ldots,f_N\}$ be the OSM feature set. Each feature $f_i$ is
represented by an axis-aligned MBR, longitude and latitude extents
$w_i^{\mathrm{lon}}$ and $w_i^{\mathrm{lat}}$, and centroid $c_i$.
BR-LZ first maps each centroid to a $B$-bit Morton key
$z_i=\operatorname{Morton}_B(c_i)$ and sorts features by this key. The
sorted order induces a permutation $\pi$, where $z_{\pi(j)}$ is monotone
in rank $j$. The sorted array is then partitioned into $S$ equi-count
segments. For each segment $k$ with rank interval $I_k=[p_k,p_{k+1})$,
BR-LZ fits an endpoint-interpolating linear model
$\hat{\pi}_k(z)=a_kz+b_k$ and stores its maximum rank residual
\begin{equation}
\rho_k=\max_{j\in I_k}
\left|\hat{\pi}_k(z_{\pi(j)})-j\right|.
\end{equation}
BR-LZ also stores a maximum half-extent for each segment:
\begin{equation}
h^{\mathrm{lon}}_k=\max_{j\in I_k} w_{\pi(j)}^{\mathrm{lon}}/2,\qquad
h^{\mathrm{lat}}_k=\max_{j\in I_k} w_{\pi(j)}^{\mathrm{lat}}/2.
\end{equation}
These local extents are no larger than the global maximum extents and are
used to reduce over-expansion during query processing.

Given a query MBR $q$, BR-LZ enumerates the Morton-key segments that may
contain overlapping features. For each segment $k$, it expands $q$ by the
segment-local extent $(h^{\mathrm{lon}}_k,h^{\mathrm{lat}}_k)$, obtains the
corresponding key interval, and uses the residual bound $\rho_k$ to convert
that key interval into a rank window. Every candidate in the rank window is
then checked by the exact MBR predicate. The learned model is therefore
used only to generate a safe candidate window; correctness is enforced by
the residual bound and the final exact filter. Figure~\ref{alg:brlz} gives the complete BR-LZ build and query procedure.

\begin{figure}[t]
\centering
\fbox{\parbox{0.93\columnwidth}{\small
\textbf{Algorithm 1.}~BR-LZ build and range query\\[2pt]
\textbf{Build}$(F,B,S)$:\\
\hspace*{1em}1: \textbf{for} each $f_i\in F$: $z_i\gets\operatorname{Morton}_B(c_i)$\\
\hspace*{1em}2: $\pi\gets\operatorname{argsort}(z)$; relabel $F$ in $z$-order\\
\hspace*{1em}3: split the sorted array into $S$ equi-count segments\\
\hspace*{1em}4: \textbf{for} each segment $k$:\\
\hspace*{2em}fit $\hat{\pi}_k(z)=a_kz+b_k$\\
\hspace*{2em}$\rho_k\gets\max_{j\in I_k}|\hat{\pi}_k(z_{\pi(j)})-j|$\\
\hspace*{2em}store local extents
$h^{\mathrm{lon}}_k,h^{\mathrm{lat}}_k$\\[3pt]
\textbf{Query}$(q)$:\\
\hspace*{1em}1: $\mathcal{A}\gets\emptyset$\\
\hspace*{1em}2: \textbf{for} each segment $k$ whose key range overlaps $q$:\\
\hspace*{2em}$q_k\gets$ expand $q$ by
$(h^{\mathrm{lon}}_k,h^{\mathrm{lat}}_k)$\\
\hspace*{2em}$[z_{\min},z_{\max}]\gets$ Morton-key interval of $q_k$\\
\hspace*{2em}$\ell\gets\hat{\pi}_k(z_{\min})-\rho_k$;\quad
$u\gets\hat{\pi}_k(z_{\max})+\rho_k$\\
\hspace*{2em}\textbf{for} $j\in[\ell,u)$:
\textbf{if} $f_j.\mathrm{MBR}\cap q\neq\emptyset$
\textbf{then} $\mathcal{A}\mathrel{+}=f_j$\\
\hspace*{1em}3: \textbf{return} $\mathcal{A}$
}}
\caption{BR-LZ build and query. The bounded residual $\rho_k$ converts a
predicted key interval into a safe rank window, while the final MBR filter
preserves exact range semantics.}
\label{alg:brlz}
\end{figure}

\begin{figure}[t]
\centering
\includegraphics[width=0.80\columnwidth]{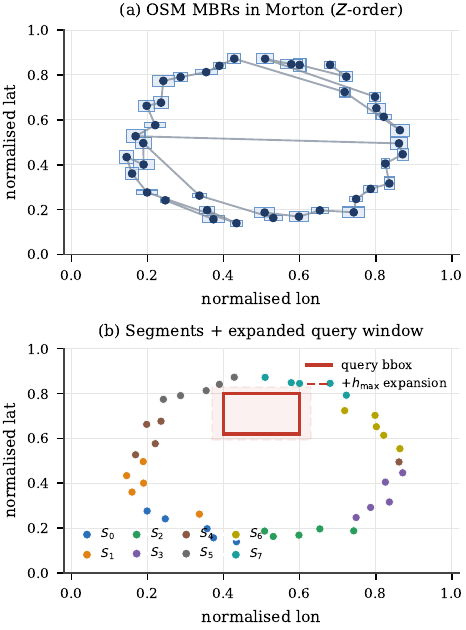}
\caption{BR-LZ linearises feature centroids by Morton order and partitions
the sorted array into equi-count segments. Segment-local extents expand the
query only where needed, and bounded residuals provide a safe rank window.}
\label{fig:brlz-schematic}
\end{figure}

\subsection{Recall completeness}
\label{ssec:brlz-recall}

\begin{theorem}[BR-LZ recall completeness]
\label{thm:brlz-recall}
For any query MBR $q$ and any feature $f_i\in F$, if
$f_i.\mathrm{MBR}\cap q\neq\emptyset$, then BR-LZ returns $f_i$.
\end{theorem}

\begin{proof}
Let $f_i$ belong to segment $k$. If $f_i.\mathrm{MBR}$ overlaps $q$, then
the centroid $c_i$ lies inside $q$ expanded by the half-extent of $f_i$.
Since $h^{\mathrm{lon}}_k$ and $h^{\mathrm{lat}}_k$ upper-bound the
half-extents of all features in segment $k$, $c_i$ lies inside the
segment-expanded query $q_k$. Therefore its Morton key $z_i$ is covered by
the key interval enumerated for segment $k$. By construction of the
bounded residual,
$|\hat{\pi}_k(z_i)-\pi^{-1}(i)|\le\rho_k$, so the true rank
$\pi^{-1}(i)$ lies in the scanned rank window. The final exact MBR filter
then accepts $f_i$ because $f_i.\mathrm{MBR}\cap q\neq\emptyset$.
\end{proof}

The theorem turns recall from an empirical tuning target into a structural
property. It holds for any segment count $S$ and Morton resolution $B$.
This distinction matters in a learning pipeline: a missed map feature does
not merely reduce query recall, but changes the context tensor on which
the downstream model is trained.

\subsection{Segment-local extent and candidate amplification}
\label{ssec:extent}

Recall completeness can also be achieved by expanding every query with a
single global maximum feature extent. BR-LZ instead stores the maximum
extent per segment. Since the segment-local extent is bounded above by the
global extent, the segment-expanded query is never larger than the
globally expanded query for that segment, and it can be substantially
smaller in spatially localised regions, as shown in Figure~\ref{fig:brlz-schematic}.

We quantify this effect using candidate amplification,
i.e., the number of candidates tested by the exact MBR filter per true
answer. Table~\ref{tab:extent} compares three expansion strategies on the
DMA corpus: a single global extent, the BR-LZ segment-local extent, and a
$95$th-percentile segment extent with an overflow list. All strategies
preserve recall $1.000$; the difference is the size of the candidate set.

\begin{table}[t]
\centering
\caption{Candidate amplification on DMA
($N{=}40{,}604$, $n{=}10$ trials). All variants preserve recall
$1.000$; segment-local expansion yields the smallest candidate set.}
\label{tab:extent}
\small
\setlength{\tabcolsep}{4pt}
\begin{tabular}{rrrrr}
\toprule
$r$ (m) & global & \textbf{segment} & quantile & segment/global \\
\midrule
$1{,}000$  & $240.1$  & $\mathbf{87.9}$ & $134.0$ & $\mathbf{0.37\times}$ \\
$3{,}000$  & $1{,}119$ & $\mathbf{616}$  & $1{,}615$ & $\mathbf{0.55\times}$ \\
$5{,}000$  & $804$    & $\mathbf{416}$  & $1{,}086$ & $\mathbf{0.52\times}$ \\
$10{,}000$ & $273$    & $\mathbf{250}$  & $884$    & $\mathbf{0.92\times}$ \\
\bottomrule
\end{tabular}
\setlength{\tabcolsep}{6pt}
\end{table}

Segment-local expansion reduces amplification by
$1.1\times$--$2.7\times$ relative to global expansion at the same recall.
The quantile variant is less stable because the overflow list dominates at
larger radii. We therefore use segment-local expansion as the default
BR-LZ configuration.

\subsection{Complexity and empirical position}
\label{ssec:pareto}

BR-LZ builds with one sort and one pass over the sorted array, giving
$O(N\log N)$ build time and $\Theta(N+S)=\Theta(N)$ space. A query first
locates the relevant segment/key ranges and then scans residual-expanded
rank windows. With equi-count segments, setting
$S=\Theta(\sqrt{N})$ balances the number of segments and the per-segment
residual window, yielding an $O(\log N+\sqrt{N})$ query bound. Table
~\ref{tab:pareto-asym} summarises the asymptotic comparison with
representative one-curve learned indexes.

\begin{table}[t]
\centering
\caption{Asymptotic comparison of one-curve learned range indexes.}
\label{tab:pareto-asym}
\small
\setlength{\tabcolsep}{3pt}
\begin{tabular}{lccc}
\toprule
Index & Build & Query & Footprint \\
\midrule
ZM-Index
  & $O(N\log N)$
  & $O(\log S+\rho_{\max})$
  & $O(N+S)$ \\
LISA
  & $O(N\log N)$
  & $O(|G_q|\,\bar{\rho}_{\mathrm{cell}})$
  & $O(N+G^2)$ \\
\textbf{BR-LZ}
  & $O(N\log N)$
  & $O(\log N+\sqrt{N})$
  & $O(N)$ \\
\bottomrule
\end{tabular}
\setlength{\tabcolsep}{6pt}
\end{table}

Table~\ref{tab:pareto} compares BR-LZ with classical and learned baselines
on DMA under the same query harness. BR-LZ has the fastest build time and
the smallest learned-index footprint, while its query latency is within
the range of the exact classical baselines. Unlike the other learned
indexes in the table, its recall guarantee follows from
Theorem~\ref{thm:brlz-recall} rather than from empirical tuning.

\begin{table}[t]
\centering
\caption{OSM index comparison on DMA
($N{=}40{,}604$, $r{=}5$\,km, $n{=}10$ trials). All systems are timed
through one unified harness.}
\label{tab:pareto}
\small
\setlength{\tabcolsep}{3pt}
\begin{tabular}{lrrrc}
\toprule
Index & build (ms) & size (KB) & $p_{50}$ ($\mu$s) & guarantee \\
\midrule
STR-tree (ref)     & $30.4$ &  $1183$  & $65.4$ & exact \\
LibSpatial R$^*$   & $81.7$ & $3172$  & $70.9$ & exact \\
LISA               & $16.1$ & $1160$  & $66.7$ & no \\
ZM-Index           & $16.1$ & $1220$  & $85.3$ & no \\
RSMI               & $32.7$ & $1240$  & $68.8$ & no \\
Flood~\cite{nathan2020flood}
                   & $20.8$ & $1155$  & $69.2$ & no \\
LMSFC~\cite{gao2022lmsfc}
                   & $16.9$ & $1272$  & $89.6$ & no \\
\textbf{BR-LZ}
                   & $\mathbf{15.9}$ & $\mathbf{1145}$ & $63.5$
                   & \textbf{provable} \\
\bottomrule
\end{tabular}
\setlength{\tabcolsep}{6pt}
\end{table}

These properties make BR-LZ the preferred learned backend for M-CTX's
static range stage. Its compact state consists only of the Morton-sorted
feature array, segment coefficients, residual bounds, and segment-local
extents. The index serialises to a small memory-mappable blob and can be
broadcast to workers without native runtime dependencies, which is
important for the multi-worker context-construction setting evaluated in
Section~\ref{ssec:concurrent}.

\section{Linear-Time SDF and Storage Co-Design}
\label{sec:sdf}

The SDF stage is the largest component of the reference context pipeline
(Table~\ref{tab:profile}). For each anchor, it materialises two
$128\times128$ signed distance fields: distance to shore and distance to
navigable water. M-CTX improves this stage in two steps. First, it replaces
the quadratic distance kernel with an exact linear-time transform. Second,
it reduces the storage footprint of the generated SDF tensors without
changing the downstream model. Figure~\ref{fig:sdf-context} shows the emitted SDF context tile for an
anchor-centred patch.

\begin{figure}[t]
\centering
\includegraphics[width=0.96\columnwidth]{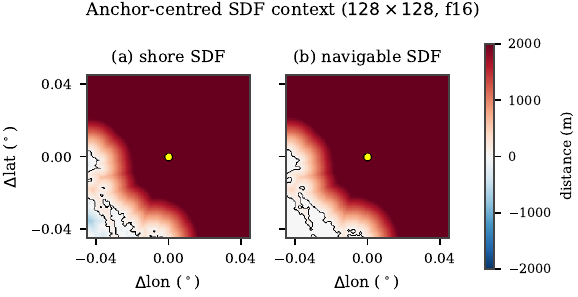}
\caption{Anchor-centred SDF context tile. M-CTX emits signed distance to
shore and navigable water over the local $5$\,km patch.}
\label{fig:sdf-context}
\end{figure}

\subsection{Exact linear-time SDF construction}

The reference \texttt{\_udist} kernel computes an unsigned distance field
by materialising pairwise distances from every grid cell to every mask
pixel and taking a minimum. For an $H\times W$ grid with $M$ occupied mask
pixels, this costs $\Theta(HW\cdot M)$ per SDF. M-CTX replaces this kernel
with the exact Felzenszwalb--Huttenlocher Euclidean distance transform
~\cite{felzenszwalb2004distance}, which computes the same distance field
in $\Theta(HW)$ time using two separable lower-envelope passes. We use the
compiled CPU implementation~\cite{virtanen2020scipy} and provide an
optional batched GPU variant for large corpus rebuilds.

\begin{table}[t]
\centering
\caption{Per-scene SDF latency for $80$ patches per scene at grid size
$g=128$. The linear-time engine is occupancy-invariant, while the
quadratic baseline slows down as mask density increases.}
\label{tab:sdf-scene}
\small
\setlength{\tabcolsep}{3.5pt}
\begin{tabular}{lrrrr}
\toprule
Scene & Occ.\ (px) & Naive (ms) & M-CTX (ms) & Speed-up \\
\midrule
Harbor       & $2066.5$ & $2340.4$ & $1.20$ & $1\,955\times$ \\
Nearshore    & $780.2$  & $1443.5$ & $1.03$ & $1\,402\times$ \\
Constrained  & $474.0$  & $2135.8$ & $1.01$ & $\mathbf{2\,109\times}$ \\
Open water   & $5.5$    & $0.16$   & $0.15$ & $1.1\times$ \\
\bottomrule
\end{tabular}
\setlength{\tabcolsep}{6pt}
\end{table}

Table~\ref{tab:sdf-scene} shows that the speed-up follows the expected
asymptotic behaviour. Dense scenes benefit most because the quadratic
baseline scales with the number of mask pixels, whereas M-CTX is nearly
constant across occupancy. At the full-stage level, SDF latency drops from
$176.4$\,ms to $1.08$\,ms per anchor, a $163\times$ reduction. At higher
resolutions, the gap widens further: at $g=512$, the quadratic kernel
exceeds $5$\,s per sample, while M-CTX remains below $23$\,ms.

For very large batches, the optional GPU path stacks patches along the
batch dimension and amortises kernel-launch overhead. On our hardware, it
reduces batch-$64$ latency from $0.84$\,ms with the CPU path to
$0.37$\,ms. Since the CPU implementation already removes SDF construction
from the critical path, the GPU path is used only for large-scale
re-materialisation.

\subsection{Storage--accuracy co-design}
\label{ssec:storage}

Exact SDF construction removes the computational bottleneck, but storing
two dense SDF rasters per anchor can still dominate the corpus footprint.
We therefore sweep dtype, resolution, and clipping thresholds, and evaluate
each representation using the pretrained \texttt{lstm\_env\_sdf} model on
$n=2{,}000$ test anchors. Table~\ref{tab:storage} reports the no-clipping
frontier, measured against the $128^2$ f16 baseline. Figure~\ref{fig:storage} plots the corresponding storage--accuracy Pareto
frontier.

\begin{table}[t]
\centering
\caption{SDF storage--accuracy frontier without clipping. Compression is
measured against the $128^2$ f16 baseline.}
\label{tab:storage}
\small
\setlength{\tabcolsep}{3pt}
\begin{tabular}{lrrr}
\toprule
Representation & B/anchor & Compression & $\Delta$ADE (m) \\
\midrule
f16 $\times$ $128^2$ baseline & $65\,536$ & $1\times$  & $0.000$ \\
f16 $\times$ $64^2$           & $16\,384$ & $4\times$  & $-0.031$ \\
uint8 $\times$ $64^2$         & $8\,192$  & $8\times$  & $-0.023$ \\
f16 $\times$ $32^2$           & $4\,096$  & $16\times$ & $+0.031$ \\
\textbf{uint8 $\times$ $32^2$}
                              & $\mathbf{2\,048}$
                              & $\mathbf{64\times}$
                              & $\mathbf{+0.038}$ \\
\bottomrule
\end{tabular}
\setlength{\tabcolsep}{6pt}
\end{table}

\begin{figure}[t]
\centering
\includegraphics[width=0.9\columnwidth]{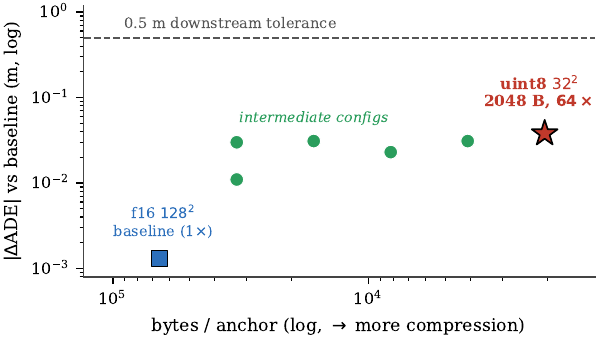}
\caption{SDF storage--accuracy Pareto front. Quantisation and downsampling
provide large footprint reductions with negligible ADE change.}
\label{fig:storage}
\end{figure}

The best operating point is \textbf{uint8 quantisation at
$32\times32$ resolution}: it stores each anchor in $2{,}048$ bytes,
achieving $64\times$ compression with only a $+0.038$\,m ADE change. At
corpus scale, this reduces the SDF working set from roughly $720$\,GB to
about $11$\,GB, turning a multi-machine materialisation problem into a
single-machine one.

The sweep also reveals a clear failure mode. Downsampling and uniform
quantisation have negligible impact, whereas narrow-band clipping is
harmful: clipping the SDF magnitude to $\pm500$\,m or $\pm1$\,km increases
ADE by tens of metres. This suggests that the downstream model tolerates
low-resolution and low-precision SDF storage, but still relies on the
long-range distance magnitude. The design rule is therefore simple:
\emph{compress by dtype and resolution, but do not saturate the distance
range}.

\subsection{Precision sensitivity}
\label{ssec:precision}

The storage frontier in Table~\ref{tab:storage} hides an important
distinction: reducing dtype or resolution is safe, but saturating the SDF
magnitude is not. Table~\ref{tab:sdf-prec} reports a seven-level precision
sweep using the pretrained \texttt{lstm\_env\_sdf} checkpoint on
$n=2{,}000$ test anchors. Downsampling to $32^2$ and 8-bit quantisation
leave ADE essentially unchanged, even when tensor-level SDF error grows.
In contrast, narrow-band clipping increases ADE by tens of metres. This
shows that the model is tolerant to coarse and low-precision SDF storage,
but still uses the long-range distance magnitude. The resulting rule is
therefore: compress by dtype and resolution, but do not clip the distance
range.

\begin{table}[t]
\centering
\caption{SDF precision sensitivity. $\Delta$ADE is measured against the
f32 reference representation on $n=2{,}000$ test anchors.}
\label{tab:sdf-prec}
\small
\setlength{\tabcolsep}{4pt}
\begin{tabular}{llrr}
\toprule
Level & Representation & SDF MAE (m) & $\Delta$ADE (m) \\
\midrule
L0 & f32 reference           & $0.0$    & $0.000$ \\
L1 & f16 storage             & $0.0$    & $0.000$ \\
L4 & $64^2$ resolution       & $0.6$    & $+0.008$ \\
L5 & $32^2$ resolution       & $1.4$    & $+0.020$ \\
L6 & 8-bit quantisation      & $52.5$   & $-0.003$ \\
L3 & clip $\pm500$\,m        & $4{,}175$ & $+66.5$ \\
L2 & clip $\pm1$\,km         & $3{,}697$ & $+103.7$ \\
\bottomrule
\end{tabular}
\setlength{\tabcolsep}{6pt}
\end{table}

\section{The Streaming Neighbour Index}
\label{sec:neighbor}

The neighbour stage answers $K_S(a,3\text{\,km},k)$ for each anchor. The
reference pipeline scans the full vessel snapshot, costing
$88.2$\,ms per anchor. M-CTX replaces this scan with an incremental
B$^{x}$-tree~\cite{jensen2004bx}, reducing the neighbour query to
$14.2\,\mu$s and yielding a $6{,}212\times$ stage-level speed-up.

\subsection{Incremental moving-object lookup}

Each vessel record is encoded as a phase-aware spatial key
\[
k(t,x,y)=\mathrm{phase}(t)\,\Vert\,\operatorname{Morton}_B(x,y),
\qquad
\mathrm{phase}(t)=\lfloor t/T\rfloor ,
\]
where $T$ is the phase length and $\operatorname{Morton}_B$ maps the local
projected coordinates to a $B$-bit space-filling-curve key. A query first
maps the anchor-centred search window to key intervals in the relevant
phase, scans the corresponding slice, and then applies an exact geographic
radius filter before returning the top-$k$ neighbours. The key-range scan
is therefore used only to generate candidates; recall is enforced by the
final exact filter.

This design matches AIS streams for two reasons. First, inserts are
$O(\log N)$ and do not require rebuilding a snapshot index. A new time
phase simply occupies a later key range, while expired phases are evicted
from the rolling window. Second, queries and inserts can interleave at
sub-snapshot granularity, which is the operating mode of live AIS feeds.
In contrast, a snapshot KD-tree can be efficient after rebuilding, but it
cannot answer queries between two rebuilds without exposing stale state.

\subsection{Streaming performance}

Table~\ref{tab:stream} stresses one B$^{x}$-tree with $10^7$ synthetic
records under four arrival patterns and with real AIS streams from four
regions. Across all settings, recall remains $1.000$. Synthetic ingest
sustains $117$--$131$K records/s with $q_{p50}\le0.13$\,ms, and real AIS
streams sustain $214$--$273$K records/s with $q_{p50}\le0.03$\,ms. The
index is insensitive to arrival order because out-of-order records are
inserted into their corresponding phase rather than triggering a global
rebuild. Figure~\ref{fig:streaming} visualises the sustained ingest rate and query
latency across the four arrival patterns.

\begin{table}[t]
\centering
\caption{B$^{x}$-tree streaming ingest. Synthetic streams contain
$10^7$ records under four arrival patterns; real streams are evaluated per
region and after merging all four regions. Recall is $1.000$ throughout.}
\label{tab:stream}
\small
\setlength{\tabcolsep}{4pt}
\begin{tabular}{lrrr}
\toprule
Stream & Rate (rec/s) & Insert ($\mu$s) & $q_{p50}$ (ms) \\
\midrule
Synthetic, batch        & $131\,283$ & $9.91$  & $0.09$ \\
Synthetic, per-record   & $127\,583$ & $10.62$ & $0.11$ \\
Synthetic, bursty       & $117\,862$ & $11.57$ & $0.13$ \\
Synthetic, out-of-order & $117\,395$ & $12.10$ & $0.13$ \\
\midrule
Real DMA                & $214\,274$ & --      & $0.02$ \\
Real NOAA               & $272\,502$ & --      & $0.02$ \\
Real Norway             & $238\,424$ & --      & $0.02$ \\
Real Piraeus            & $243\,295$ & --      & $0.02$ \\
\textbf{Real merged-4}  & $\mathbf{253\,574}$ & -- & $0.03$ \\
\bottomrule
\end{tabular}
\setlength{\tabcolsep}{6pt}
\end{table}

\begin{figure}[t]
\centering
\includegraphics[width=0.9\columnwidth]{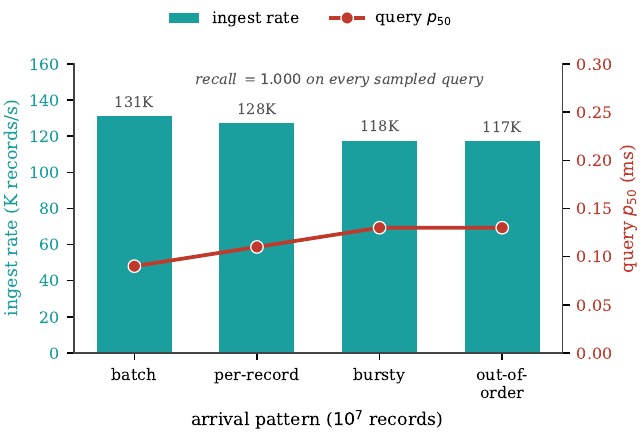}
\caption{Sustained ingest rate and query latency over $10^7$ records.
The B$^{x}$-tree remains stable across batch, per-record, bursty, and
out-of-order arrivals.}
\label{fig:streaming}
\end{figure}

\subsection{Comparison with snapshot rebuilding}

We compare B$^{x}$-tree against a fair KD-tree baseline that rebuilds once
per timestamped snapshot and projects each query into its own local ENU
frame. This removes projection bias and isolates the real design trade-off:
snapshot rebuilding versus incremental updates. As shown in
Table~\ref{tab:fair-stream}, the KD-tree has lower amortised insert cost
because many records share one rebuild, but it does not support
sub-snapshot queries. B$^{x}$-tree instead supports interleaved inserts and
queries while preserving exact recall.

\begin{table}[t]
\centering
\caption{Fair streaming-replay comparison
($50\,000$ records, $686$ snapshots, $n=3$ trials). Both methods use
per-query Haversine filtering; the distinction is update granularity.}
\label{tab:fair-stream}
\small
\setlength{\tabcolsep}{4pt}
\begin{tabular}{lrr}
\toprule
Metric & FairKDTree (rebuild) & B$^{x}$-tree (incremental) \\
\midrule
Recall vs.\ oracle  & $0.980\pm0.000$ & $\mathbf{1.000\pm0.000}$ \\
Insert ($\mu$s/rec) & $\mathbf{0.48\pm0.02}$ & $4.68\pm0.16$ \\
Sub-snapshot query  & no & \textbf{yes} \\
\bottomrule
\end{tabular}
\setlength{\tabcolsep}{6pt}
\end{table}

Figure~\ref{fig:nbr-scale} further shows that query latency remains nearly
flat across anchor count for both KD-tree and B$^{x}$-tree. Thus, the
advantage of B$^{x}$-tree is not a per-query micro-optimisation; it is the
ability to maintain an exact neighbour index under streaming updates. The
phase length is also stable in practice: varying $T$ from $10$ to $60$\,s
changes build time by less than $10\%$ and query latency by less than
$4\%$. We therefore use $T=20$\,s, matching the AIS cadence, and a
$16$-bit Morton grid by default.

\begin{figure}[t]
\centering
\includegraphics[width=\columnwidth]{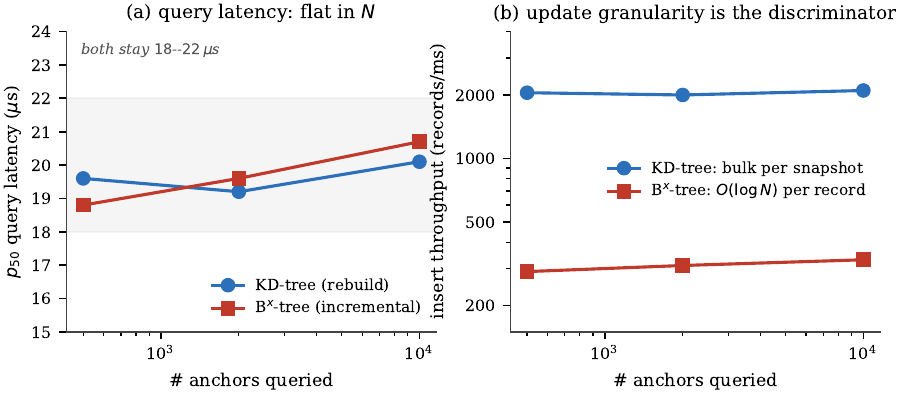}
\caption{Neighbour index scaling. Query latency remains flat across anchor
count, while B$^{x}$-tree supports per-record incremental updates required
by live AIS streams.}
\label{fig:nbr-scale}
\end{figure}

\section{Experimental Evaluation}
\label{sec:bench}

\subsection{Setup}

We evaluate M-CTX on the EnvShip-Bench Standard Track
($120$K/$15$K/$15$K train/val/test anchors)
and on OSM caches from four real maritime regions: Denmark (DMA), NOAA
East Coast, Norway, and Piraeus. The four-region OSM caches and AIS corpora are built with the publicly released EnvShip-Bench pipeline,\footnote{\url{https://github.com/mark000071/EnvShip-Bench_Large_Dataset_Pipeline_and_datasets}} and all experiments in this paper can be reproduced from the M-CTX repository,\footnote{\url{https://github.com/mark000071/M-CTX-Traj}} where each table maps to a one-line command. Table~\ref{tab:datasets} summarises the
static map workloads. Unless otherwise stated, all systems are timed on a
single Intel Xeon + NVIDIA L40S (48\,GB) node through the same per-query harness.
We report mean values over $n=10$ trials and score recall against a warm
linear-scan oracle on queries with at least one true hit.

\begin{table}[t]
\centering
\caption{Real maritime regions used for cross-region evaluation.}
\label{tab:datasets}
\small
\setlength{\tabcolsep}{4pt}
\begin{tabular}{llrr}
\toprule
Region & Coastline type & \# features & Recall \\
\midrule
DMA (Denmark)    & North Sea / Baltic & $40\,604$ & $1.000$ \\
NOAA (E.\ USA)   & Atlantic seaboard  & $89\,757$ & $1.000$ \\
Norway           & fjord / narrow     & $14\,557$ & $1.000$ \\
Piraeus (Greece) & Aegean port        & $992$     & $1.000$ \\
\bottomrule
\end{tabular}
\setlength{\tabcolsep}{6pt}
\end{table}

We compare against eight baselines: a warm in-memory linear scan,
libspatialindex R$^{*}$-tree, H3, DuckDB-spatial, LISA, RSMI,
Flood~\cite{nathan2020flood}, and LMSFC~\cite{gao2022lmsfc}. M-CTX's
STR-tree and BR-LZ are evaluated under the same harness. The experiments
answer four questions: (i) how much each context stage accelerates;
(ii) how BR-LZ compares with classical and learned range indexes;
(iii) whether the system generalises across regions and scales to large
feature sets; and (iv) how the component gains compose end-to-end.

\paragraph*{Measurement protocol and reproducibility}
To reduce measurement artefacts, all indexed range-query systems are
evaluated through the same Python-level harness and the same final exact
MBR filter. We separate cold-start costs from steady-state query costs:
index construction and data loading are reported explicitly, while
per-query latency is measured after the index has been materialised in
memory. Each table is generated from a JSON result file, and an audit
script recomputes all reported speed-ups from the corresponding source
latencies before producing the paper tables. The released artifact contains
the scripts, configuration files, and one-command runners for every
reported experiment, including the OSM index benchmarks, SDF precision
sweep, streaming replay, and end-to-end context build.

\subsection{Headline acceleration}
\label{ssec:headline}

Figure~\ref{fig:headline} reports per-stage speed-ups against a fair
baseline for each stage: warm linear scan for OSM range retrieval, the
reference quadratic distance transform for SDF, and brute-force snapshot
scan for neighbours. M-CTX reduces OSM range latency from
$1\,446\,\mu$s to $63.5\,\mu$s ($23\times$), SDF latency from
$176.4$\,ms to $1.08$\,ms ($163\times$), and neighbour lookup from
$88.2$\,ms to $14.2\,\mu$s ($6{,}212\times$).

\begin{figure}[t]
\centering
\includegraphics[width=\columnwidth]{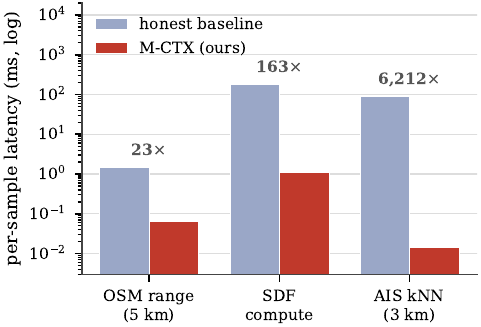}
\caption{Per-stage latency reduction. M-CTX accelerates OSM range
retrieval, SDF construction, and neighbour lookup by $23\times$,
$163\times$, and $6{,}212\times$, respectively.}
\label{fig:headline}
\end{figure}

These stage-level gains compose into a large end-to-end improvement.
Table~\ref{tab:e2e} shows that M-CTX reduces the $150$K-anchor Standard
Track context build from about $11$ hours to $169$ seconds. On the full
$5.48$M-anchor corpus, the measured full-pipeline runtime drops from
about $17$ CPU-days to $1.8$ hours, yielding the headline $226\times$
speed-up.

\begin{table}[t]
\centering
\caption{End-to-end context-build time. The $5.48$M-anchor row is a
measured full-pipeline run; the $150$K row composes the measured
per-stage costs.}
\label{tab:e2e}
\small
\setlength{\tabcolsep}{4pt}
\begin{tabular}{lrrr}
\toprule
Workload & Reference & M-CTX & Speed-up \\
\midrule
150K anchors
  & $39\,690$\,s ($\approx$11\,h)
  & $169$\,s
  & $235\times$ \\
5.48M anchors
  & $1.45$M\,s ($\approx$17\,d)
  & $6\,420$\,s ($\approx$1.8\,h)
  & $\mathbf{226\times}$ \\
\bottomrule
\end{tabular}
\setlength{\tabcolsep}{6pt}
\end{table}

\subsection{OSM range-retrieval comparison}
\label{ssec:cross-system}

Table~\ref{tab:osm-cross} compares BR-LZ with classical, external, and
learned range indexes on DMA at $r=5$\,km. All exact systems are evaluated
through the same wrapper and final MBR filter. Every indexed backend is
substantially faster than the warm linear scan. BR-LZ achieves the lowest
$p_{50}$ latency, the fastest build time among the compared indexes, and
recall $1.000$ with a formal recall-completeness guarantee.

\begin{table}[t]
\centering
\caption{Cross-system OSM range-query benchmark on DMA
($N=40\,604$, $r=5$\,km, $n=10$ trials). Recall is measured against the
linear-scan oracle.}
\label{tab:osm-cross}
\small
\setlength{\tabcolsep}{3pt}
\begin{tabular}{lrrrr}
\toprule
System & Build (ms) & $p_{50}$ ($\mu$s) & QPS & Recall \\
\midrule
Linear scan (warm)   & --     & $1\,446$ & $0.7$\,k & $1.000$ \\
\midrule
M-CTX STR-tree       & $30.4$ & $65.4$   & $15$\,k  & $1.000$ \\
libspatialindex      & $81.7$ & $70.9$   & $14$\,k  & $1.000$ \\
H3 (res 7, $k=4$)    & $50$   & $73.4$   & $14$\,k  & $0.997$ \\
DuckDB spatial       & $30$   & $508$    & $1.9$\,k & $1.000$ \\
\midrule
LISA                 & $16.1$ & $66.7$   & $15$\,k  & $1.000$ \\
RSMI                 & $32.7$ & $68.8$   & $14$\,k  & $1.000$ \\
Flood~\cite{nathan2020flood}
                     & $20.8$ & $69.2$   & $14$\,k  & $1.000$ \\
LMSFC~\cite{gao2022lmsfc}
                     & $16.9$ & $89.6$   & $9.9$\,k & $1.000$ \\
\textbf{BR-LZ}
                     & $\mathbf{15.9}$ & $\mathbf{63.5}$
                     & $\mathbf{16}$\,k & $\mathbf{1.000}$ \\
\bottomrule
\end{tabular}
\setlength{\tabcolsep}{6pt}
\end{table}

The result is important for the M-CTX setting because build time and
footprint are part of the amortised cost. BR-LZ does not merely match
classical indexes in query latency; it also provides the fastest cold
construction and the smallest learned-index footprint reported in
Section~\ref{ssec:pareto}.

\subsection{Cross-region robustness}
\label{ssec:cross-region}

To test distribution shift, we repeat the OSM benchmark on structurally
different coastlines. Table~\ref{tab:region} reports Norway, a narrow-fjord
region with footprint-corrected queries. All indexes retain recall
$1.000$, and BR-LZ gives the fastest build and lowest latency at both
query radii. The full four-region results show the same pattern: recall
remains exact without per-region tuning, despite a $90\times$ spread in
feature count and substantial differences in coastline geometry.

\begin{table}[t]
\centering
\caption{Cross-region OSM index ranking on Norway
($N=14\,557$, recall $1.000$).}
\label{tab:region}
\small
\setlength{\tabcolsep}{4pt}
\begin{tabular}{lrrr}
\toprule
Index & Build (ms) & $p_{50}$@2km ($\mu$s) & $p_{50}$@5km ($\mu$s) \\
\midrule
STR-tree     & $7.9$  & $136$ & $345$ \\
LISA         & $6.5$  & $159$ & $286$ \\
RSMI         & $7.6$  & $234$ & $404$ \\
LibSpatial   & $30.3$ & $142$ & $211$ \\
\textbf{BR-LZ}
              & $\mathbf{6.2}$ & $\mathbf{112}$ & $\mathbf{139}$ \\
\bottomrule
\end{tabular}
\setlength{\tabcolsep}{6pt}
\end{table}

\subsection{Scaling to large static maps}
\label{ssec:scale}

Table~\ref{tab:scale} scales the static OSM workload to $40$M synthetic
features. Build time grows approximately linearly after sorting dominates,
and all indexes remain below $10$\,ms median query latency at $40$M
features. BR-LZ preserves its build advantage across two orders of
magnitude and remains among the fastest query backends at every scale. Figure~\ref{fig:scale-synth} shows the corresponding median-latency trend
as the feature count increases.

\begin{table}[t]
\centering
\caption{Synthetic OSM scale-up to $N=40$M features. Each cell reports
build time (s) / $p_{50}$ query latency (ms), with $r=5$\,km.}
\label{tab:scale}
\small
\setlength{\tabcolsep}{3pt}
\begin{tabular}{lrrrrr}
\toprule
$N$ & STR & LISA & ZM & RSMI & BR-LZ \\
\midrule
1M
  & 0.8/0.34  & 0.6/0.14  & 0.5/0.15  & 0.6/0.17   & 0.5/0.16 \\
4M
  & 3.4/0.72  & 2.5/0.34  & 2.4/0.36  & 2.7/0.52   & 2.4/0.35 \\
16M
  & 14.4/2.55 & 13.1/1.04 & 11.1/1.36 & 14.3/2.48  & 13.1/1.09 \\
40M
  & 37.3/5.93 & 31.0/2.50 & 28.3/3.21 & 31.1/10.16 & 28.5/2.57 \\
\bottomrule
\end{tabular}
\setlength{\tabcolsep}{6pt}
\end{table}

\begin{figure}[t]
\centering
\includegraphics[width=0.95\columnwidth]{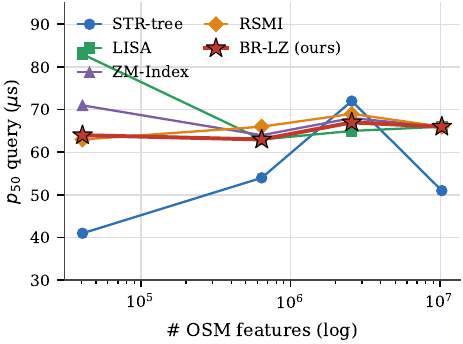}
\caption{Synthetic OSM scale-up. Median query latency remains stable as
feature count grows by two orders of magnitude.}
\label{fig:scale-synth}
\end{figure}

\subsection{Concurrent throughput and cold start}
\label{ssec:concurrent}

M-CTX parallelises context retrieval by partitioning anchor batches across
workers while sharing broadcast-loaded indexes. Table~\ref{tab:concurrent}
evaluates STR-tree throughput on the merged four-region corpus
($N=145\,910$). Throughput increases from $14{,}623$ qps with one worker
to $174{,}900$ qps with $16$ workers, an $11.96\times$ speed-up. A
space-tiled shard partition keeps the mean number of shards touched per
query below $1.08$, indicating that M-CTX can be partitioned cleanly for
multi-node execution.

\begin{table}[t]
\centering
\caption{Concurrent STR-tree throughput on the merged four-region corpus
($N=145\,910$, $30\,000$ queries, $n=3$ trials).}
\label{tab:concurrent}
\small
\setlength{\tabcolsep}{4pt}
\begin{tabular}{rrrr}
\toprule
Workers & QPS & Speed-up & Efficiency \\
\midrule
$1$  & $14\,623$  & $1.00\times$  & $1.00$ \\
$2$  & $25\,375$  & $1.73\times$  & $0.87$ \\
$4$  & $49\,853$  & $3.41\times$  & $0.85$ \\
$8$  & $116\,079$ & $7.94\times$  & $0.99$ \\
$16$ & $\mathbf{174\,900}$ & $\mathbf{11.96\times}$ & $0.75$ \\
\bottomrule
\end{tabular}
\setlength{\tabcolsep}{6pt}
\end{table}

Cold-start overhead is small after data loading. On the Standard Track,
tile loading and JSON deserialisation take $355.8$\,s, while MBR extraction
takes $1.0$\,s and STR-tree construction only $90.5$\,ms. Amortised over
$150$K anchors, the one-time startup cost is about $2.4$\,ms per anchor and
is paid once rather than per query.

\subsection{Partitioning and Deployment Robustness}
\label{ssec:partition}

The single-node implementation broadcasts one index copy to multiple
workers. To study whether the same design can be sharded for distributed
execution, we simulate a space-tiled deployment on the merged four-region
corpus. Figure~\ref{fig:shard-part} compares two layouts. Morton-rank
stripes balance feature count, but their spatial footprints overlap
heavily, causing broad query fan-out. A kd-tree median split
~\cite{bentley1975kdtree} instead produces spatially disjoint cells, so a
query is routed only to the cells it intersects. In our simulator, this
partition keeps the mean number of shards touched per query below $1.08$
up to $16$ shards. Figure~\ref{fig:shard} quantifies this query fan-out under space-tiled sharding. Thus, M-CTX can be partitioned without turning each
context request into a broadcast query; measuring inter-node latency and
replication cost is left to a full distributed deployment.

\begin{figure}[t]
\centering
\includegraphics[width=0.96\columnwidth]{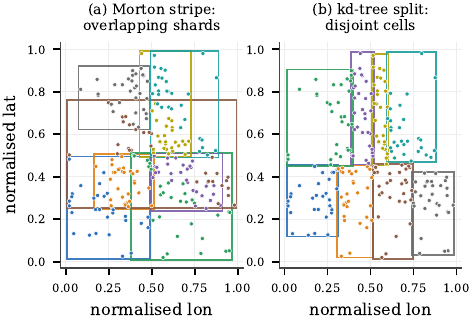}
\caption{Shard partitioning. Morton-rank stripes balance key ranges but
create overlapping spatial footprints; kd-tree median splits produce
disjoint cells and lower query fan-out.}
\label{fig:shard-part}
\end{figure}

\begin{figure}[t]
\centering
\includegraphics[width=0.88\columnwidth]{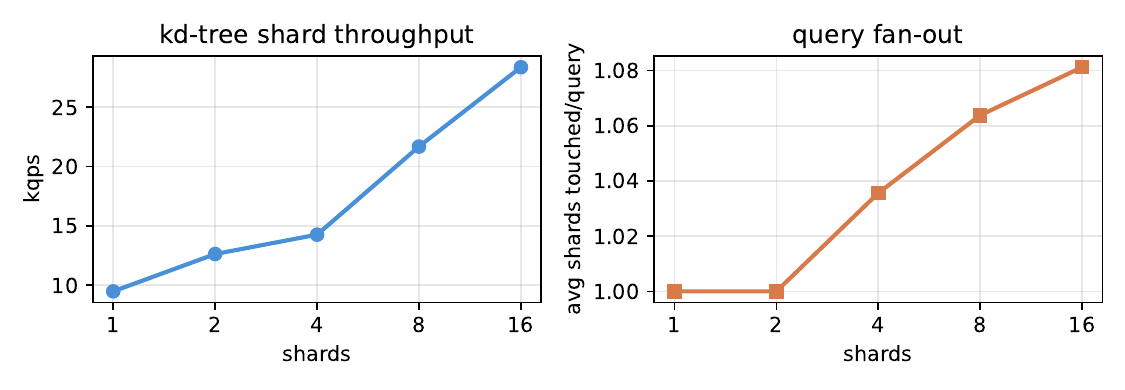}
\caption{Space-tiled shard simulator. Mean query fan-out remains below
$1.08$ up to $16$ shards, supporting clean partitioning for distributed
context retrieval.}
\label{fig:shard}
\end{figure}

\subsection{Component ablation}
\label{ssec:ablation}

Table~\ref{tab:e2e_ablation} decomposes the end-to-end gain by replacing
one or more reference stages with the corresponding M-CTX component.
SDF acceleration is the largest single contributor, reducing the
Standard-Track build from $11.6$ hours to $4.3$ hours. Neighbour indexing
alone gives a $1.5\times$ gain, while OSM-only acceleration is small
because tile loading is amortised in this setting. Combining SDF and
neighbour indexing already reaches $19.6\times$; the full system reduces
the build to $169$ seconds.

\begin{table}[t]
\centering
\caption{End-to-end component ablation on the $150$K-anchor Standard
Track. Each row replaces one or more reference stages with M-CTX
components.}
\label{tab:e2e_ablation}
\small
\setlength{\tabcolsep}{3pt}
\begin{tabular}{lcccrr}
\toprule
Variant & OSM & SDF & $k$NN & Total & Speed-up \\
\midrule
Reference      & ref   & ref   & ref   & $11.6$\,h   & $1.0\times$ \\
OSM-only       & M-CTX & ref   & ref   & $11.0$\,h   & $1.05\times$ \\
SDF-only       & ref   & M-CTX & ref   & $4.3$\,h    & $2.7\times$ \\
$k$NN-only     & ref   & ref   & M-CTX & $7.9$\,h    & $1.5\times$ \\
OSM+SDF        & M-CTX & M-CTX & ref   & $3.7$\,h    & $3.1\times$ \\
OSM+$k$NN      & M-CTX & ref   & M-CTX & $7.4$\,h    & $1.6\times$ \\
SDF+$k$NN      & ref   & M-CTX & M-CTX & $35.5$\,min & $19.6\times$ \\
\textbf{Full M-CTX}
               & \textbf{M} & \textbf{M} & \textbf{M}
               & $\mathbf{169}$\,\textbf{s}
               & $\mathbf{235\times}$ \\
\bottomrule
\end{tabular}
\setlength{\tabcolsep}{6pt}
\end{table}

All headline gaps are statistically significant under a paired Wilcoxon
signed-rank test at the trial level ($n=10$), with matched query sets
across systems. We additionally verify every displayed speed-up against
its two source measurements using an automated audit script.

\section{Downstream Equivalence}
\label{sec:downstream}

M-CTX is useful only if acceleration does not change the context consumed
by the predictor. We therefore verify equivalence at two levels: the
emitted context tensors and the downstream predictions.

\paragraph*{Context-level equivalence}
For SDF construction, the reference \texttt{\_udist} kernel and the
M-CTX linear-time EDT produce identical float32 distance fields, with mean
absolute element-wise difference $0.000$\,m over the
$128\times128\times2$ tensor. When compared against the stored float16
raster, the residual is $0.089$\,m; after both outputs are cast to the
stored float16 representation, the difference becomes $0.000$\,m. Thus,
the remaining discrepancy is due to storage quantisation, not algorithmic
divergence. For OSM range retrieval and neighbour lookup, M-CTX returns
the same feature and vessel identifier sets as the oracle at recall
$1.000$. Under the reference storage schema, M-CTX therefore emits
byte-identical context tensors.

\paragraph*{Model-level equivalence}
On the full $15{,}000$-sample test set, replacing the reference context
backend with M-CTX reproduces the published ADE/FDE
($85.8271$/$211.1504$\,m) with per-step prediction MAE exactly $0$.
Table~\ref{tab:multi-model} further evaluates five pretrained checkpoints
on $n=2{,}000$ test samples. In all cases, M-CTX yields the same ADE and
zero prediction MAE under the reported precision. Thus, the acceleration
introduces no measurable downstream accuracy cost, satisfying the
drop-in-equivalence requirement of a production context backend.

\begin{table}[t]
\centering
\caption{Downstream impact across five pretrained checkpoints
($n=2{,}000$ test samples). M-CTX preserves the predictions of all tested
models under the reported precision.}
\label{tab:multi-model}
\small
\setlength{\tabcolsep}{5pt}
\begin{tabular}{lrrr}
\toprule
Model & ADE (m) & $\Delta$ADE (m) & Pred.\ MAE (m) \\
\midrule
LSTM+SDF          & $76.7165$ & $0$ & $0$ \\
LSTM+spatial-attn & $78.2553$ & $0$ & $0$ \\
LSTM+binary-attn  & $94.5108$ & $0$ & $0$ \\
LSTM+social       & $77.0156$ & $0$ & $0$ \\
Transformer+SDF   & $76.9342$ & $0$ & $0$ \\
\bottomrule
\end{tabular}
\setlength{\tabcolsep}{6pt}
\end{table}

\section{Discussion and Limitations}
\label{sec:discuss}

\paragraph*{Deployment choices}
M-CTX exposes a small set of backends rather than a single fixed index,
because different deployments prioritise different costs. For static OSM
retrieval, BR-LZ is the recommended learned backend: it combines the
fastest build time, the smallest learned-index footprint, competitive
query latency, and a recall-completeness guarantee for MBR-overlap
queries. The STR-tree remains a strong exact alternative when a deployment
prefers a dependency-free classical index. For SDF materialisation, the
best operating point is stable: \texttt{uint8} at $32\times32$ resolution,
without magnitude clipping, gives $64\times$ compression with only a
$0.04$\,m ADE change. For neighbour retrieval, the B$^{x}$-tree is the
appropriate backend for live streams because it supports interleaved
inserts and queries; snapshot KD-trees are competitive only for fully
static batches. These choices let M-CTX trade build time, query latency,
memory footprint, and formal guarantees explicitly.

\paragraph*{Limitations}
M-CTX is designed for read-mostly spatial context. The OSM stage assumes
that map features are ingested once and queried many times; workloads with
frequent feature insertions or deletions would require an update-friendly
range index, which we leave to future work. The B$^{x}$-tree assumes a
regular snapshot cadence; sub-second streams may require finer phase
management to avoid excessive key churn. Our experiments are single-node
and in-memory. We evaluate shard routing with a space-tiled simulator, but
a full distributed deployment would need to measure inter-node latency,
replication, and failure recovery. Finally, the $40$M-feature stress test
uses deterministic spatial replication of real OSM features. A
planet-scale, memory-mapped OSM deployment would further validate the
scaling curve under realistic storage pressure.

These limitations do not affect the main claim of the paper: for the
dominant read-mostly context-construction workload in trajectory learning,
M-CTX preserves the reference context exactly while reducing preprocessing
cost by orders of magnitude.

\section{Conclusion}
\label{sec:conclusion}

This paper identified spatial context retrieval as a dominant but largely
hidden cost in context-aware trajectory learning. M-CTX recasts this cost
as an ingest-once, query-many spatial database workload and replaces
brute-force preprocessing with exact, index-backed context operators. The
result is a drop-in framework that preserves the reference context schema
and semantics while reducing full-corpus context construction from about
$17$ CPU-days to $1.8$ hours, a measured $226\times$ end-to-end speed-up.

The core lesson is that model-facing context retrieval requires both
systems efficiency and retrieval correctness. BR-LZ provides
recall-complete MBR-overlap range retrieval with a tighter segment-local
candidate expansion than global-expansion one-curve baselines. The
linear-time SDF engine removes the largest computational bottleneck, and
the storage co-design compresses SDF tensors by $64\times$ with negligible
downstream impact. Across real maritime regions, learned and classical
baselines, synthetic scale-up to $40$M features, and $10^7$-record AIS
streams, M-CTX preserves downstream predictions while making context
construction practical at corpus scale.

More broadly, M-CTX shows that the retrieval layers hidden inside modern
spatiotemporal learning pipelines are first-class database problems. As
trajectory models increasingly depend on maps, distance fields, and
neighbouring agents, exact and scalable context retrieval will become as
important as the predictor itself. We release M-CTX as an open drop-in
backend for reproducible, large-scale trajectory analytics.

\bibliographystyle{IEEEtran}
\bibliography{refs}

\clearpage
% \appendices
% \input{sections/appendix}

\end{document}